\documentclass{article}

\usepackage[main, final]{neurips_2026}

\usepackage[utf8]{inputenc} %
\usepackage[T1]{fontenc}    %
\usepackage{hyperref}       %
\usepackage{url}            %
\usepackage{booktabs}       %
\usepackage{amsfonts}       %
\usepackage{nicefrac}       %
\usepackage{microtype}      %
\usepackage{xcolor}         %
\usepackage{wrapfig}

\title{Bayesian Preference Learning for\\Test-Time Steerable Reward Models}

\usepackage{amsmath}
\usepackage{amssymb}
\usepackage{mathtools}
\usepackage{amsthm}

\usepackage{hyperref}
\usepackage{url}
\usepackage{graphicx}
\usepackage{amsthm}
\usepackage{amssymb}
\usepackage{tcolorbox}
\usepackage{enumitem}
\usepackage{subcaption}
\usepackage{dsfont}
\usepackage{booktabs}
\usepackage{graphicx}
\usepackage{wrapfig}
\usepackage{longtable,array}
\usepackage{tabularray}

\usepackage[capitalize,noabbrev]{cleveref}

\theoremstyle{plain}
\newtheorem{theorem}{Theorem}[section]

\newtheorem{lemma}[theorem]{Lemma}

\theoremstyle{definition}

\theoremstyle{remark}

\newcommand{\ie}{\emph{i.e.,} }
\newcommand{\eg}{\emph{e.g.,} }

\author{%
  Jiwoo Hong$^{1}$\thanks{Equal Contribution.}
  \qquad
  Shao Tang$^{1}$\footnotemark[1]
  \qquad
  Zhipeng Wang$^{2}$ \\[1em]
  $^{1}$Nubank \qquad $^{2}$Google \\
  \texttt{\{jiwoo.hong, tang.shao\}@nubank.com.br},~~\texttt{zhipengbayern@gmail.com}
}

\begin{document}

\maketitle

\begin{abstract}
    Reward models are central to aligning language models with human preferences via reinforcement learning (RL). As RL is increasingly applied to settings such as verifiable rewards and multi-objective alignment, RMs are expected to encode more complex and multifaceted preference distributions. However, classifier RMs remain static once trained, limiting their adaptability at test time. We propose \textbf{Variational In-Context Reward Modeling (ICRM)}, a novel Bayesian reward modeling objective that enables \emph{test-time steerability} via in-context preference demonstrations. ICRM casts reward modeling as amortized variational inference over a latent preference probability under the Bradley-Terry model using a conjugate Beta prior. We show that ICRM adapts to unseen preference distributions at test time for both single and multi-objective settings. With more demonstrations, ICRM improves RM-Bench accuracy from 60.5 to 70.8, achieves lower calibration error than a generative judge on moral dilemma preferences, and expands the attainable Pareto frontier under conflicting preferences. We further study the practical applicability of ICRM for RL training, showing that it can effectively encode verifiable rewards by outperforming a conventional RM in math reasoning. Finally, we provide theoretical guarantees that the variational objective admits a global interior optimum with finite confidence, followed by an ablation study on its role for the test-time steerability, and we analyze how KL regularization mitigates reward over-optimization.
\end{abstract}

\section{Introduction}

Reward models (RMs) serve as essential proxies for human preferences in language model post-training, including reinforcement learning with human feedback (RLHF) \citep{ziegler2020finetuning, ouyang2022training, stiennon2022learning}. Specifically, triplets comprising a prompt, a preferred response, and a dispreferred response are used to parameterize the preference distribution under the Bradley–Terry (BT) model \citep{btmodel}. Neural classifiers, \ie classifier RMs, act as estimators of the BT strength parameter, with theoretical guarantees that, given sufficient preference data, the learned RM can approximate the ``true'' human preference distribution \citep{bong2022generalized,rafailov2023direct}. This formulation enables the learned RM to act as a standalone proxy for a single concatenation of prompt and response, which is practically useful for RLHF training.

However, classifier RMs face two data-driven limitations: (1) they are \emph{static} once trained on a given dataset, and (2) they are prone to over-optimization \citep{gao2023scaling,hong2025on}. While LLM-as-a-Judge \citep{kim2024prometheus} offers flexible evaluation criteria with strong performance \citep{lambert_rewardbench_2024, malik2025rewardbench2advancingreward, liu2025rmbench}, these gains often rely on proprietary models \citep{comanici2025gemini25pushingfrontier,openai2024gpt4ocard}, implying substantial compute and data costs. Hence, it is desirable to design an efficient classifier RM that remains adaptable to unseen data while avoiding over-optimization by being test-time steerable.

In this paper, we introduce a \textbf{variational in-context reward modeling (ICRM)} framework grounded in a Bayesian view of preferences. Our method approximates the true preference distribution with a Beta posterior conditioned on in-context preference demonstrations. In detail, placing a Beta prior on the BT model yields a closed-form training loss via variational inference. This variational loss enables ICRM to learn preferences \emph{in-context} with few-shot demonstrations, allowing \textbf{test-time steerability of a classifier RM} that can dynamically adapt to one or more mixtures of arbitrary preferences, as in Figure \ref{fig:main}. Furthermore, we prove that a KL penalty to the Beta prior tempers the learned preference mean and yields a global interior optimum. Our main contributions are summarized below:
\begin{enumerate}[left=1pt,parsep=1pt]
    \item \textbf{Principled variational preference learning}: We propose a reward modeling objective that enables RMs to infer test-time preferences from in-context preference samples. We prove that regularizing the Beta posterior with a uniform Beta prior guarantees a global interior optimum, tempering excessive maximization of the preference mean on training data.

    \item \textbf{Single and multi-objective test-time steerability}: We show that the variational design yields both directional and calibrated steerability, and extends to multi-objective preference trade-offs. With an increasing number of demonstrations, ICRM improves RM-Bench accuracy by over 10\% points and increases hypervolume by 4\% on conflicting RM-Bench Safety subsets.

    \item \textbf{Parameterizing verifiable rewards with ICRM}: We apply ICRM to reinforcement learning with verifiable rewards (RLVR), where it outperforms both the verifier and the Bradley-Terry reward model on math reasoning, achieving over 2.5\% point accuracy compared to the gold verifier.
\end{enumerate}

\begin{figure*}[t!]
    \centering
    \includegraphics[width=\textwidth]{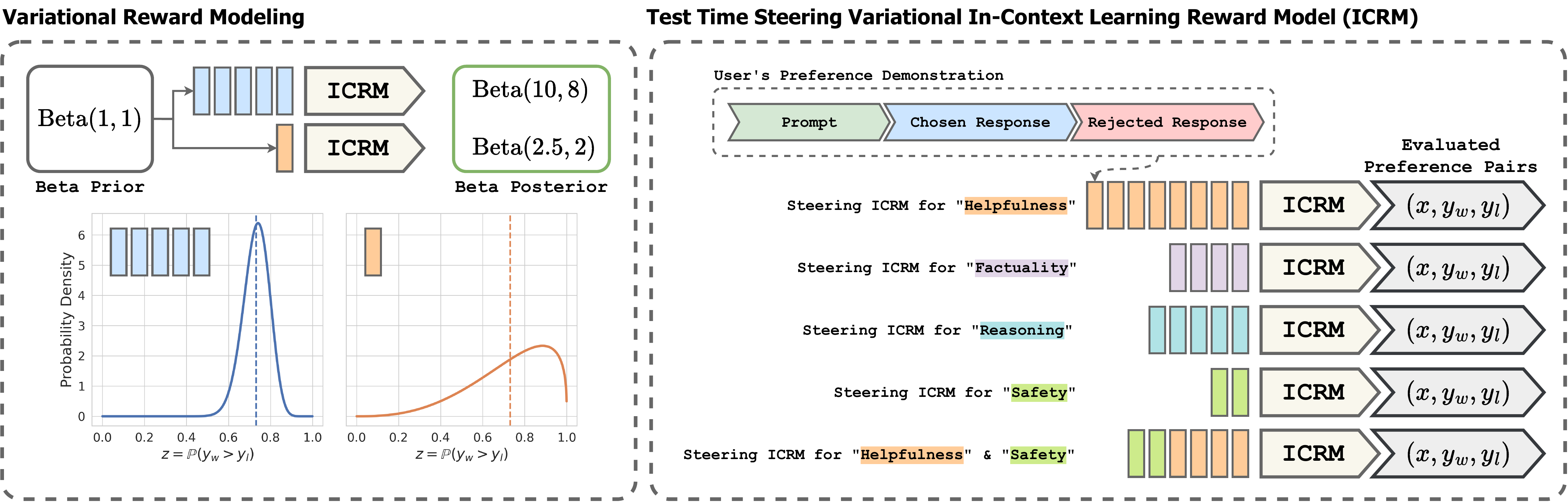}
    \caption{Variational in-context reward modeling (\textbf{ICRM}) with Beta prior for the Bradley-Terry (BT) model. ICRM directly encodes the mean and sharpness of the Beta posterior given in-context preference demonstrations, allowing \textbf{multi-objective test-time steerability} for \emph{any} unseen preferences.}
    \label{fig:main}
\end{figure*}

\section{Background}\label{sec:background}

\paragraph{Preliminaries.} A classifier reward model (RM), $r_\theta(x, y)$, is a function parameterized by $\theta$ that outputs a score indicating the quality of a prompt-response pair $(x, y)$ \citep{ziegler2020finetuning}:
\begin{equation}
    r_\theta(x, y) = W_p^\top h_\theta(x, y) \in \mathbb{R},\label{eq:rm}
\end{equation}
where $W_p \in \mathbb{R}^{d_\mathrm{model} \times 1}$ is a projection head initialized by $\mathcal{N}(0, (d_\mathrm{model}+1)^{-1})$ \citep{stiennon2022learning,huang2024the,hong2025on} and $h_\theta(x, y) \in \mathbb{R}^{d_\mathrm{model} \times 1}$ is the last hidden state from the backbone language model. These models are typically trained on a dataset of human preferences, $\mathcal{D} = \{ (x_i, y_{i,w}, y_{i,l}) \}_{i=1}^N$, where $y_{i,w}$ is the preferred (``chosen'') response and $y_{i,l}$ is the dispreferred (``rejected'') response for a given prompt $x_i$. The training objective maximizes the log-likelihood of the preferences according to the Bradley-Terry (BT) model \citep{btmodel},
\begin{equation}
    P(y_w \succ y_l \mid x) = \sigma(r_\theta(x, y_w) - r_\theta(x, y_l)) = \frac{\exp(r_\theta(x, y_w))}{\exp(r_\theta(x, y_w)) + \exp(r_\theta(x, y_l))},
\end{equation}
which posits that the probability of $y_w$ being preferred over $y_l$ is given by a logistic function of the difference in their reward scores. The final loss function $\mathcal{L}_\mathrm{BT}(\theta)$ is defined as:
\begin{equation}
    \mathcal{L}_\mathrm{BT}(\theta) = - \left( \log \sigma(r_\theta(x, y_w) - r_\theta(x, y_l)) \right).
\end{equation}
Once the preference distribution shown in the training set is encoded into $\theta$ via fine-tuning, it cannot be adaptively updated at test time without additional training, limiting the flexibility of RMs.

\paragraph{Estimating the true preference distribution.}
Prior work in offline preference learning supports that, with sufficient pairwise comparisons, fitted models recover underlying preferences \citep{rafailov2023direct, hejna2024contrastive}. In the classical BT setting, the maximum-likelihood estimator (MLE) exists and enjoys consistency and asymptotic normality. For any context \(x\) and pair \((y_w,y_l)\), if $\hat{P}$ is the probability estimated by the MLE and \(P^{*}\) the true probability, then
$\hat{P}(y_w \succ y_l \mid x) \xrightarrow{p} P^*(y_w \succ y_l \mid x).$ Thus, with sufficient data, a learned BT model converges to the \emph{true} preference distribution. We study the parameterization of the BT model and its applications in Appendix \ref{apdx:related_works}.

\paragraph{Bayesian treatment of the Bradley-Terry model.} Bayesian treatment of the BT model necessitates the selection of a suitable prior distribution for parameters \citep{chen1984bayes,whelan2017prior,wainer2023bayesian,fageot2024generalized}. The general form of the model with $N > 2$ contenders is parameterized by a vector of $N$ strength scores $\beta = (\beta_1, \ldots, \beta_N) \in \mathbb{R}^N$. Typically, the Bayesian formulation in those cases defines the prior distributions directly on each strength parameter: \eg Gaussian prior \citep{wainer2023bayesian} and a Dirichlet prior \citep{chen1984bayes}.

\section{Variational In-Context Reward Modeling}\label{sec:variational}

We present a novel Bayesian reward modeling objective by framing in-context reward modeling as a problem of amortized variational inference. The central idea is to approximate the true preference distribution with a Beta posterior conditioned on in-context preference demonstrations.

\subsection{Problem Setup}\label{subsec:setup}
\paragraph{Prior.} We assume that demonstrations $\mathcal{C}$ reflect a shared preference criterion that determines the preference probability for a target pair $(y_w,y_l)$ given prompt $x$. We introduce a latent random variable $z$ representing this target-specific probability, \ie $z \coloneqq P(y_w \succ y_l \mid x,\mathcal{C}) \in [0,1]$. We represent uncertainty about $z$ through a true but intractable \emph{context-dependent} prior, $p(z \mid x,y_w,y_l,\mathcal{C})$. Demonstrations from different prompts provide contextual evidence about the shared criterion, rather than repeated Bernoulli observations of the target pair. Conditioned on $z$, the likelihood of the observed outcome $o \in \{0,1\}$, $p(o \mid z,x,y_w,y_l,\mathcal{C})$, follows a Bernoulli distribution.

\paragraph{Posterior.} By Bayes' rule, the \emph{true posterior} over $z$ after observing $o$, \ie our inferential target is:
\begin{equation}
p(z \mid o, \cdot)
    \propto p(o \mid z, \cdot) \times p(z \mid \cdot),
\label{eq:true_posterior}
\end{equation}
where $(x, y_w, y_l, \mathcal{C})$ is omitted. However, this is intractable as the context-dependent prior $p(z \mid x, y_w, y_l, \mathcal{C})$ lacks a simple analytical form due to the complex dynamics of in-context learning. Throughout, we focus on $o = \mathds{1}_{y_w \succ y_l}$. Thus, we approximate the posterior distribution through $q_\theta(z \mid o = \mathds{1}_{y_w \succ y_l}, x, y_w, y_l,  \mathcal{C})$, which is denoted as $q_\theta(z \mid x, y_w, y_l, \mathcal{C})$ for notational brevity. 

\subsection{Reward Modeling as Variational Inference}\label{subsec:icrm}

We parameterize $q_\theta(z \mid x, y_w, y_l,  \mathcal{C})$ with model $\theta$, which directly maps the inputs to the parameters of an approximate posterior distribution, namely the variational in-context reward modeling (ICRM). We outline the choice of the prior distribution and propose the final learning objective for ICRM.

\paragraph{Beta prior for the Bradley-Terry model.} Building on the discussion of the Bayesian treatment of the BT model, we propose using a Beta prior for the BT model in reward modeling. The setting for reinforcement learning from human feedback (RLHF) typically involves a \emph{single} pairwise comparison $(y_w, y_l)$ given the prompt $x$ \citep{wang2024helpsteer,liu2025skywork}. This specialization to $N = 2$ significantly reduces the problem's complexity, \ie likelihood of observing preference outcomes for this pair follows a Bernoulli distribution parameterized by $z$. 
For a Bernoulli likelihood, the conjugate prior for the parameter is the Beta distribution: $\mathrm{Beta}\left(\alpha_0, \beta_0\right)$, where $(\alpha_0, \beta_0)$ encodes our initial belief about the preference before observing any data.

\paragraph{Amortized variational approximation of posterior.} 
Given the Beta prior, we approximate the posterior distribution $q_\theta(z \mid x, y_w, y_l, \mathcal{C})$ using a reward model with a two-dimensional projection head $W_p \in \mathbb{R}^{d_\mathrm{model} \times 2}$, returning a \textit{utility} score $u_\theta(x,y, \mathcal{C})$ and a \textit{confidence} (\ie evidence) score $s_\theta(x,y, \mathcal{C})$, which are context dependent. For $(x, y_w,y_l, \mathcal{C})$, we have both scores, each for the chosen and rejected responses $y_w$ and $y_l$, shortened as $u_w, u_l, s_w$, and $s_l$. We reparameterize the Beta posterior $\mathrm{Beta}(\alpha_q, \beta_q)$ with $\alpha_q = \mu \tau$ and $\beta_q = (1-\mu)\tau$, where
\begin{equation}
\mu = \sigma(u_w - u_l)~~\text{and}~~\tau = \mathrm{Softplus}(s_w) + \mathrm{Softplus}(s_l) + 1,
\label{eq:variational_parameters}
\end{equation}
with $\mathrm{Softplus}(x) = \log (1 + \exp(x))$. Here $\mu\in(0,1)$ is the posterior predictive probability and $\tau>0$ controls concentration. The approximate posterior is
$q_\theta(z\mid x,y_w,y_l,\mathcal{C})
=\mathrm{Beta}\bigl(z;\alpha_q,\beta_q\bigr)$, with $\alpha_q > 0$ and $\beta_q>0$.
This construction preserves the BT model as a special case: the posterior mean of $q_\theta$ recovers the BT preference probability, $\mathbb{E}_{q_\theta}[z] = \alpha_q/(\alpha_q+\beta_q)= \mu = \sigma(u_w - u_l),$ while the concentration $\tau$ reflects the amount of evidence.

\paragraph{Weighted evidence lower bound for variational objective.} Since the true posterior $p(z \mid x, y_w, y_l, \mathcal{C})$ is intractable as described in Section \ref{subsec:setup}, we formulate the inference task as an optimization problem using variational inference to approximate the true posterior with the reward model $r_\theta$. 
Inspired by \citet{joo2020being}, we train the model $\theta$ by maximizing the weighted Evidence Lower Bound (ELBO) for the observed preference $y_w \succ y_l$:
\begin{equation}
      \mathcal{L}_{\mathrm{ELBO}}(\theta)= -\underbrace{\mathbb{E}_{q_\theta(z \mid x, y_w, y_l, \mathcal{C})}\left[\log z\right]}_{\text{Reconstruction Error}} +
  \lambda(N)\times \underbrace{\mathbb{D}_\mathrm{KL}\left(q_\theta \left(z \mid \cdot \right) \,\mid\mid\, p\left(z \mid \cdot\right)\right)}_{\text{Regularization Term}},
\label{eq:elbo}
\end{equation}
where $\cdot$ in the regularization term omits $(x, y_w, y_l, \mathcal{C})$ and $\lambda(N) = \lambda \times N^{-1}$ with predefined $\lambda$. At $\lambda(N)=1$, it equals the negative ELBO for the Bernoulli-Beta model, while other weights yield a tempered variational objective. The first term in \eqref{eq:elbo} represents the \textit{reconstruction error}, measuring how well the approximate posterior explains the observed outcome $y_w \succ y_l$. For a Beta distribution, this expectation has a closed-form solution involving the digamma function, $\psi(x):= d \log \Gamma (x)/ dx $:
\begin{equation}
 \mathbb{E}_{q_\theta(z \mid x, y_w, y_l, \mathcal{C})}\left[\log z\right] = \psi(\alpha_q) - \psi(\alpha_q + \beta_q)=\psi(\mu\tau) - \psi(\tau). 
\end{equation}
Minimizing this term increases $\mu$ toward 1, favoring $y_w$, analogous to the standard BT loss \citep{azar2024general}. Meanwhile, $\tau$ controls how sharply the distribution concentrates around this preference.

The second term in \eqref{eq:elbo} is the Kullback-Leibler (KL) divergence from the model's approximate posterior $q_\theta = \mathrm{Beta}(\mu\tau, (1-\mu) \tau)$ to the prior $p$. As the true prior $p(z \mid x, y_w, y_l, \mathcal{C})$ is intractable, we replace it with a fixed, uninformative prior $p(z) = \mathrm{Beta}(z; \alpha_0, \beta_0)$, \eg a uniform prior with $\alpha_0=\beta_0=1$. And $\lambda(N)$ is a monotonically decreasing schedule
that down-weights the KL term as the amount of contextual evidence $N$ grows. This term regularizes the approximation, preventing the posterior from deviating excessively from the prior, especially when contextual evidence is minimal, \eg $N$ is small. The KL divergence between two Beta distributions, $p = \mathrm{Beta}(\alpha_p, \beta_p)$ and $q = \mathrm{Beta}(\alpha_q, \beta_q)$, has a closed form solution \citep{wallach2019continuous}:
\begin{equation}
\begin{split}
    \mathbb{D}_\mathrm{KL}(q \,\mid\mid\, p) &= \log\frac{\Gamma(\alpha_q+\beta_q)}{\Gamma(\alpha_q)\Gamma(\beta_q)} - \log\frac{\Gamma(\alpha_p+\beta_p)}{\Gamma(\alpha_p)\Gamma(\beta_p)} + (\alpha_q - \alpha_p)[\psi(\alpha_q) - \psi(\alpha_q+\beta_q)] \\
    &~~ + (\beta_q - \beta_p)[\psi(\beta_q) - \psi(\alpha_q+\beta_q)],
\end{split}
\label{eq:kl_beta}
\end{equation}
Finally, the dynamic hyperparameter $\lambda(N)$ controls this balance: when the context is minimal ($N=1$), a large $\lambda(1)$ forces the posterior to remain close to the uninformative prior, \ie high uncertainty. As more examples are added to the context, $\lambda(N)$ decreases, allowing the reconstruction term to dominate and the model to form a more confident, data-driven posterior distribution. Combining these components, the fully-specified loss is defined as:
\begin{equation}
\mathcal{L}_{\text{ICRM}}(\mu, \tau; \alpha_0, \beta_0) = - \left( \psi(\mu \tau) - \psi(\tau) \right)+ \lambda(N) \mathbb{D}_\mathrm{KL}\left(\mathrm{Beta}(\mu\tau, (1-\mu)\tau) \| \mathrm{Beta}(\alpha_0, \beta_0)\right),
\label{eq:final_icrm_loss_expanded}
\end{equation}
where $\mu, \tau$ are functions of $\theta$. For notational convenience, we henceforth write $\mathcal{L}_{\text{ICRM}}(\mu, \tau)$. The schedule $\lambda(N)$ weights the KL penalty during training; at inference, model parameters remain fixed, and steering is achieved by conditioning on $\mathcal{C}$.

\paragraph{Choice of uniform Beta prior for the divergence penalty.} As in \eqref{eq:final_icrm_loss_expanded}, the divergence penalty can be controlled with the predefined prior distribution $p = \mathrm{Beta}(\alpha_0, \beta_0)$. If we have explicitly collected annotations for the pair $(x, y_w, y_l)$ for given few-shot examples $\mathcal{C}$, we may set unique $(\alpha_0, \beta_0)$. However, it is typically hard to collect such data. Thus, we assume $(\alpha_0, \beta_0) = (1, 1)$, implying the uniform distribution on preferring $y_w$ over $y_l$ without any information. Potentially, synthetic personas \citep{singh2025fspo} or voting over multiple preference models \citep{yang2024voting} can be used to generate such data to provide a more informative prior.

\subsection{Analysis}\label{sec:analysis}

One common failure mode of the Bradley-Terry (BT) reward model is \emph{over-optimization} \citep{gao2023scaling}, in which the preference probabilities converge to $1$ and fit into the local optima of the \emph{true} human preference distribution \citep{azar2024general,hong2025on}. The proposed KL-regularized variational objective directly addresses this issue, \ie it precludes boundary minima—ensuring an interior optimum—and, via the same KL term, imposes a quantitative edge-behavior barrier that moderates the excessive growth of the score margin at high preference probabilities.

\begin{lemma}[Edge behavior at finite confidence]\label{lemma:edge_behavior}
Let \(P_\theta(y_w \succ y_l \mid x)\) denote the ICRM preference with \(\mu=\sigma(\Delta u_\theta)=\sigma(u_\theta(x,y_w)-u_\theta(x,y_l))\) and \(\varepsilon:=1-\mu\).
For \(\tau\in(0,\infty)\),
as \(\varepsilon\to 0^+\),
\begin{equation*}
\nabla_\theta \mathcal{L}_\mathrm{ICRM}
\rightarrow
\underbrace{\left(\tfrac{\lambda \beta_0}{\varepsilon \tau}\right)}_{\text{\normalfont Utility Coefficient}}\mkern-20mu\nabla_\theta \Delta u_\theta
+ \mkern-20mu \underbrace{\left(-\tfrac{\lambda \beta_0}{\varepsilon \tau^2}\right)}_{\text{\normalfont Confidence Coefficient}}\mkern-25mu\nabla_\theta \tau.
\end{equation*}
\end{lemma}
Here, $\Delta u_\theta$ is regularized by $\tau$. As training increases $\mu$, $\lambda\beta_0 / (\tau \varepsilon)$ in the utility coefficient increases for any finite $\tau$, thereby penalizing further growth of the utility and preventing uncontrolled maximization of $\mu$ when taking a gradient descent step (see Appendices \ref{sec:icrm-gradient} and \ref{sec: proof of lemma edge behavior} for proof). Since the Lemma \ref{lemma:edge_behavior} is for finite \(\tau\), we next prove that the global minimizer indeed has \(0<\tau^\star<\infty\).
\begin{theorem}\label{thm:icrm-interior}
Fix a finite number of demonstrations \(N\geq1\).
Assume $\lambda(N)>0$ and $\alpha_0,\beta_0>0$.
Over the free output parameters
$(\mu,\tau)\in(0,1)\times(0,\infty)$, the pointwise loss in
\eqref{eq:final_icrm_loss_expanded} admits a unique global minimizer given by
\[
\tau^\star=\alpha_0+\beta_0+\frac{1}{\lambda(N)},
\qquad
\mu^\star=
\frac{\alpha_0+1/\lambda(N)}
     {\alpha_0+\beta_0+1/\lambda(N)}.
\]
In particular, $0<\mu^\star<1$ and $0<\tau^\star<\infty$.
\end{theorem}
Consequently, this theoretically guarantees prevention of reward model over-optimization via preference mean tempering. See Appendix \ref{sec: proof of theorem icrm interior} for proof.

\section{Experiments}\label{sec:exp}
Given a \emph{single} trained ICRM, we validate the test-time steerability from two aspects: (1) directional steerability and (2) calibrated steerability. Then, we conduct a case study on a multi-objective setting.
\begin{enumerate}[left=2pt,nosep]
    \item \textbf{Calibrated Steerability (single-objective)}: Can ICRM steer its preference decisions in proportion to the preference mixture specified by in-context demonstrations $\mathcal{C}$?
    \item \textbf{Directional steerability (single-objective)}: Does the posterior predictive mean $\mathbb{E}_{q_\theta}[z] = \mu$ adapt to the implicit preference distribution induced by in-context demonstrations $\mathcal{C}$?
    \item \textbf{Multi-Objective Steerability}: Can the posterior predictive mean $\mu$ encode competing preferences with respect to the in-context demonstrations $\mathcal{C}$?
\end{enumerate}

\subsection{Training Configurations}\label{subsec:train_setup}

We use Qwen3-4B \citep{yang2025qwen3technicalreport} and Llama-3.2-3B \citep{dubey2024llama3herdmodels}. We use the pre-trained base models to prevent biased prior preference distributions from post-training. The projection head $W_p \in \mathbb{R}^{d_\mathrm{model}\times 2}$ is initialized with $\mathcal{N}(0, (d_\mathrm{model}+1)^{-1})$ \citep{huang2024the}. We use Skywork-Preferences-v0.2 \citep{liu_skywork-reward_2024}, assuming each dataset reflects a consistent implicit preference distribution. For each training instance, we construct in-context demonstrations $\mathcal{C}=\{(x,y_w,y_l)\}_{j=1}^N$ with $N\in\{1,2,4,8,16\}$, sampled within train data. We use a prompt template without any arbitrary instructions in Appendix~\ref{apdx:template} to minimize template bias. We sweep different $\lambda \in \{0.1, 0.5, 1.0\}$ and report the training dynamics in the Appendix \ref{apdx:train_config}, along with the detailed training details.

\subsection{Evaluation Configurations}

\paragraph{Calibrated steerability.}
We test whether ICRM can encode the \emph{degree} of a preference shift specified by the composition of $\mathcal{C}$ on human-generated moral dilemma preference pairs from princi/pal \citep{wang2025principal}. Given $\mathcal{C}$ drawn from SafeRLHF, we vary the preference mixture by flipping $r$ fraction of demonstrations given $N$, where $r \in \mathcal{R}=\{0,0.1,\ldots,1.0\}$. Let $\hat{p}_{\mathrm{flip}}(r)$ denote the accuracy at which the model follows the flipped preference on held-out pairs. We measure calibrated steerability with:
\begin{equation}
\mathrm{ECE}
=
\frac{1}{|\mathcal{R}|}
\sum_{r \in \mathcal{R}}
\left|
\hat{p}_{\mathrm{flip}}(r)-r
\right|,
\quad
\mathrm{RMSE}
=
\sqrt{
\frac{1}{|\mathcal{R}|}
\sum_{r \in \mathcal{R}}
\left(
\hat{p}_{\mathrm{flip}}(r)-r
\right)^2
}.
\end{equation}
Here, calibrated steerability measures how closely the proportion of held-out pairs on which ICRM follows the flipped preference matches the demonstration flipping ratio $r$.

\paragraph{Directional steerability.} We use SafeRLHF \citep{ji2023beavertails}, HHH Alignment \citep{srivastava2023beyond}, and RM-Bench \citep{liu2025rmbench}. For the first two benchmarks and four domains in RM-Bench, we randomly sample $N \in \{1, 2, 4, 8, 16, 32, 64\}$ pairs to construct $\mathcal{C}$ within each domain. We then compare the scores for $(x', y'_w)$ and $(x', y'_l)$ conditioned on $\mathcal{C}$ for preference accuracy, where $(x', y'_w, y'_l)$ is drawn from the held-out data. We report the mean and standard deviation of four runs.

\paragraph{Multi-objective steerability.} We use the \textit{Safety-Should-Respond} and \textit{Safety-Should-Refuse} subsets of the RM-Bench, corresponding to benign and adversarial prompts. We mix demonstrations from the two subsets with ratio $r \in [0,1]$, where $r$ is the proportion of helpfulness examples. For each $N$, we trace the Pareto frontier and report hypervolume \citep[HV]{zitzler1999evolutionary} against $(0.0,0.0)$.

\paragraph{Baselines.} We add five reward models (RMs) that were trained with the open-source datasets. First, we compare with three classifier RMs, Bradley-Terry \citep{liu_skywork-reward_2024}, ArmoRM \citep{wang2024armo}, and GRM \citep{yang2024regularizing}. Additionally, we test the two generative RMs: plain generative use of post-trained Qwen3-4B (``LLM-as-a-Judge'' in Table \ref{tab:merged_results}) \citep{zheng2023judge} and GRAM-r$^2$ \citep{wang2026gram}. We use the default prompt template of RM-Bench for LLM-as-a-Judge.

\begin{figure*}[t!]
    \centering
    \includegraphics[width=\textwidth]{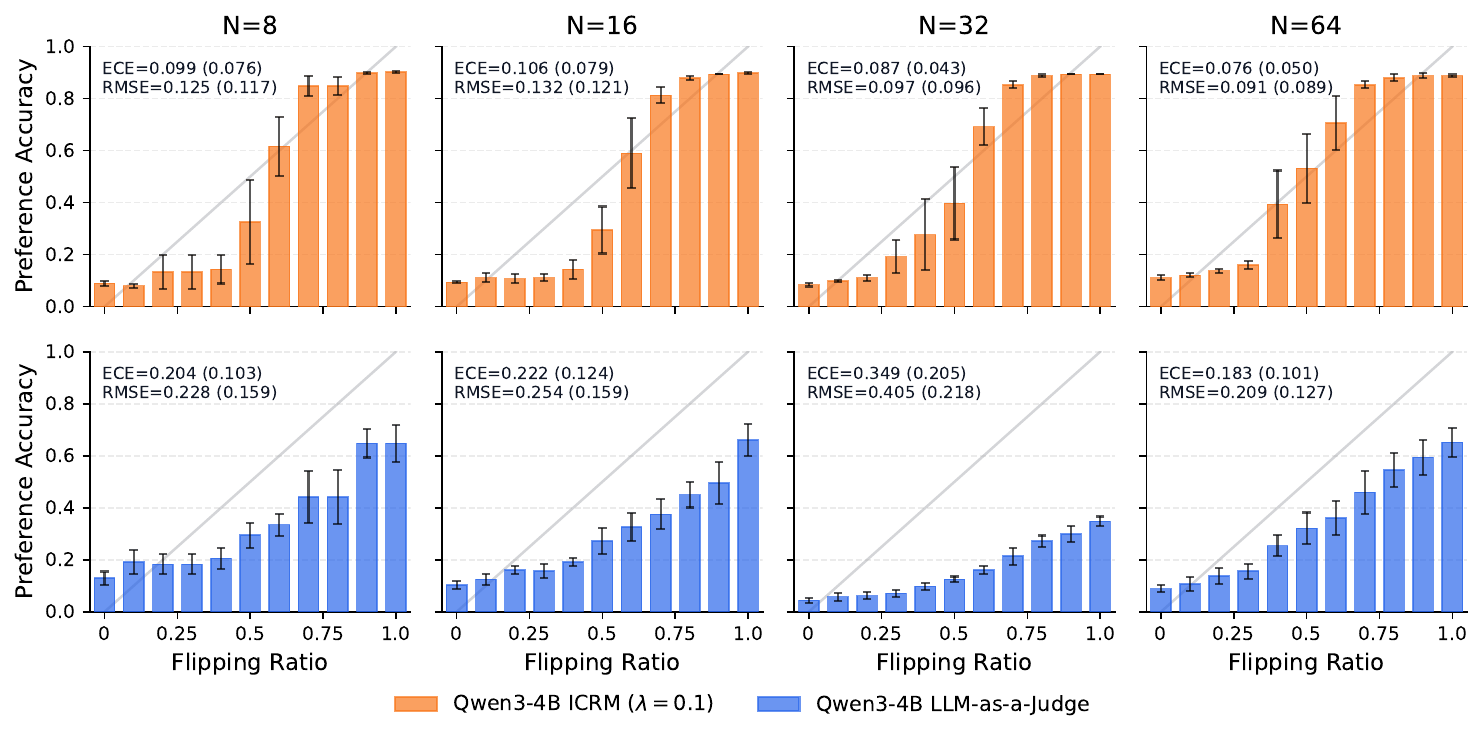}
    \vspace{-0.1in}
    \caption{\textbf{Calibrated steerability.} Calibration analysis with ECE ($\downarrow$) and RMSE ($\downarrow$) between the flipping ratio of demonstrations and the preference accuracy on moral dilemma preference pairs.}
    \label{fig:calibrated}
\end{figure*}

\section{Results}\label{sec:results}

\subsection{Calibrated Steerability (Single-Objective)}\label{subsec:calibrated}

Figure~\ref{fig:calibrated} evaluates whether ICRM can follow the degree of preference shift specified by the composition of $\mathcal{C}$. As the flipping ratio $r$ increases, ICRM consistently increases the rate of selecting the flipped preference on moral dilemma pairs, showing a monotonic response to the preference mixture. The resulting curves are closer to the ideal diagonal calibration curve than the Qwen3-4B LLM-as-a-Judge baseline across all tested context sizes. The calibration errors show the same trend. Across $N \in \{8,16,32,64\}$, ICRM obtains lower ECE and RMSE than the baseline, with ECE decreasing from $0.099$ at $N=8$ to $0.076$ at $N=64$. These results indicate that ICRM can encode graded preference mixtures at test time, rather than only steering toward a single preference direction. Based on the extended calibrated steerability analysis for different $\lambda$ and base models in Appendix \ref{apdx:calibrated}, we select Qwen3-4B ICRM with $\lambda=0.1$.

\begin{figure*}[t!]
    \centering
    \begin{subfigure}{0.495\linewidth}
        \includegraphics[width=\textwidth]{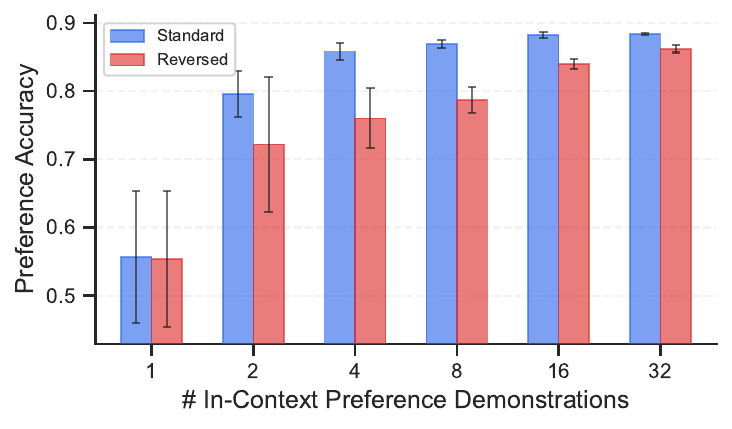}
        \caption{HHH-Alignment \citep{srivastava2023beyond}}\label{subfig:hhh}
    \end{subfigure}
    \begin{subfigure}{0.495\linewidth}
        \includegraphics[width=\textwidth]{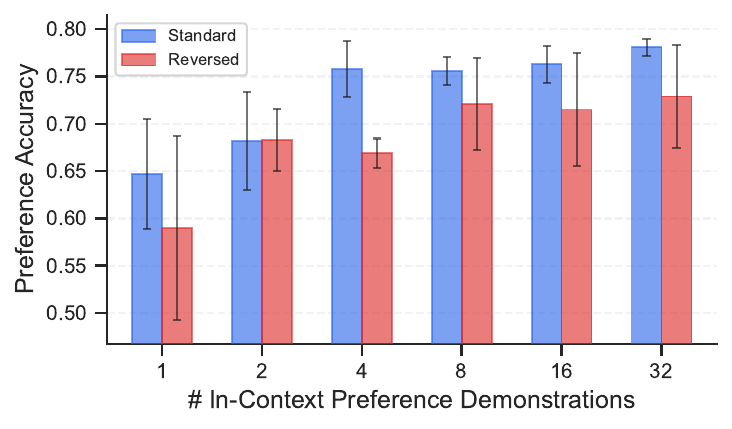}
        \caption{SafeRLHF \citep{ji2023beavertails}}\label{subfig:saferlhf}
    \end{subfigure}
    \caption{\textbf{Directional Steerability.} The mean accuracy and standard deviation of ICRM by number of in-context demonstrations from a held-out set on HHH-Alignment and SafeRLHF.}
    \label{fig:directional}
\end{figure*}
\begin{table*}[t!] %
\caption{RM-Bench evaluation results with the standard deviation of four random seeds for ICRM.}
\centering
\small

\begin{tabular}{@{}l|ccccc@{}}
\toprule
\multicolumn{1}{c|}{} & \multicolumn{5}{c}{\textbf{RM-Bench}} \\
\multicolumn{1}{c|}{} & \textbf{Chat} & \textbf{Safety} & \textbf{Code} & \textbf{Math} & \textbf{Avg.} \\ 
\midrule
\textbf{Generative} & & & & & \\
~LLM-as-a-Judge & 69.2 & 90.9 & 54.3 & \textbf{67.6} & \underline{70.3} \\
~GRAM-r$^2$ & 61.1 & 92.9 & 54.7 & 61.6 & 67.6 \\
\midrule
\textbf{Classifier} & & & & & \\
~BTRM & 69.3 & \textbf{96.0} & 53.2 & 62.1 & 70.2 \\
~ArmoRM & 67.8 & 92.4 & 53.1 & 57.5 & 67.7 \\
~GRM & 62.7 & 90.0 & \textbf{57.8} & \underline{62.5} & 68.2 \\ 
\midrule
\textbf{ICRM (\textit{Ours})} & & & & & \\
$\quad N=1$ & 48.2$_{12.6}$ & 85.6$_{15.4}$ & 50.2$_{6.5}$ & 58.0$_{10.1}$ & 60.5$_{3.5}$ \\
$\quad N=2$ & 58.0$_{3.2}$ & 91.3$_{0.9}$ & 50.8$_{5.4}$ & 59.3$_{2.0}$ & 64.9$_{2.0}$ \\
$\quad N=4$ & 59.5$_{3.2}$ & 92.0$_{0.9}$ & 53.3$_{5.8}$ & 58.0$_{1.8}$ & 65.7$_{2.0}$ \\
$\quad N=8$ & 66.6$_{1.9}$ & 91.7$_{0.2}$ & 53.7$_{6.1}$ & 58.5$_{1.1}$ & 67.6$_{1.6}$ \\
$\quad N=16$ & 64.5$_{1.8}$ & 92.4$_{0.9}$ & 54.3$_{2.5}$ & 58.7$_{0.4}$ & 67.4$_{0.4}$ \\
$\quad N=32$ & \underline{69.8}$_{3.1}$ & 91.1$_{4.1}$ & 55.0$_{0.7}$ & 60.3$_{0.7}$ & 69.1$_{0.9}$ \\ 
$\quad N=64$ & \textbf{70.1$_{3.0}$} & \underline{95.4}$_{0.5}$ & \underline{55.4$_{2.0}$} & 62.1$_{1.1}$ & \textbf{70.8$_{1.0}$} \\ 
\bottomrule
\end{tabular}
\label{tab:merged_results}
\end{table*}

\subsection{Directional Steerability (Single-Objective)}

\paragraph{General human preference benchmark.} In Table \ref{tab:merged_results} with increasing numbers of in-context demonstrations $N \in \{ 1, 2, 4, 8, 16, 32, 64 \}$, we generally observe a monotonic increase in the preference accuracy across the domains. For instance, the ``Chat'' domain gains 21.9\% from $N=1$ to $N=64$, enhancing the average RM-Bench score by around 10.3\% in total. We additionally evaluate static BTRM and GRM with the same test-time demonstrations as ICRM in Appendix \ref{apdx:static_rm_context}, showing that the static RMs did not gain from the in-context demonstrations by default, unlike ICRM.

\paragraph{Safety and helpfulness benchmarks.} Figure~\ref{fig:directional} evaluates directional steerability on safety and helpfulness benchmarks by testing both the original pairwise labels (``Standard'') and flipped labels (``Reversed''). This measures whether ICRM can adapt to the preference direction specified by $\mathcal{C}$. Across both benchmarks, ICRM increasingly follows the reversed preference direction as $N$ grows, reaching up to 89\% accuracy even when the demonstrations favor more harmful or dishonest responses. Since the classifier RM baselines are fixed after training, they cannot adapt to these reversed standards, highlighting ICRM's test-time steerability under extreme preference shifts.

\paragraph{In-context preference adaptation.} Figure \ref{fig:directional} and Table \ref{tab:merged_results} show that additional demonstrations generally improve preference accuracy. These results demonstrate that ICRM adapts its pairwise decisions to the preference criteria specified by $\mathcal{C}$, consistent with its context-conditioned formulation.

\subsection{Multi-Objective Steerability}\label{sec:multi}
\begin{figure*}[t!]
    \centering
    \begin{subfigure}{0.325\linewidth}
        \includegraphics[width=\textwidth]{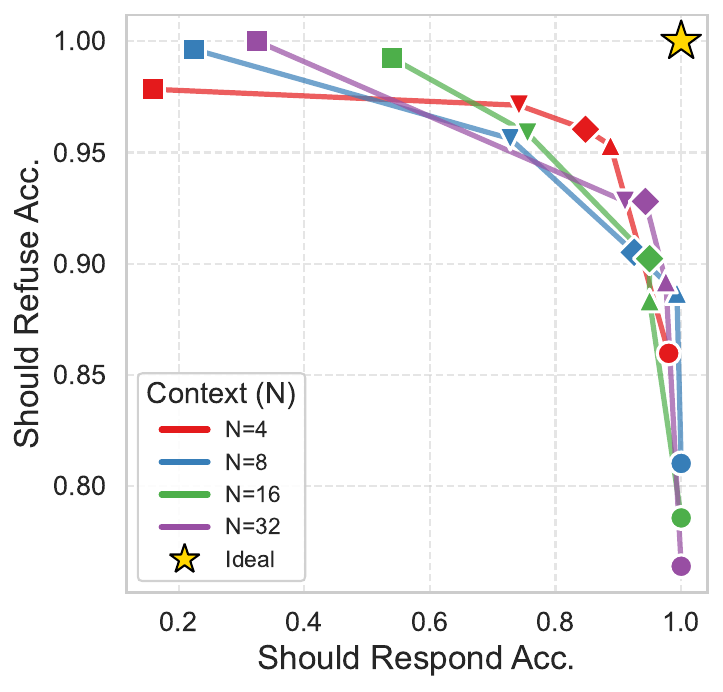}
        \caption{Pareto Frontier (Llama-3.2-3B)}\label{subfig:moo_l32}
    \end{subfigure}
    \begin{subfigure}{0.325\linewidth}
        \includegraphics[width=\textwidth]{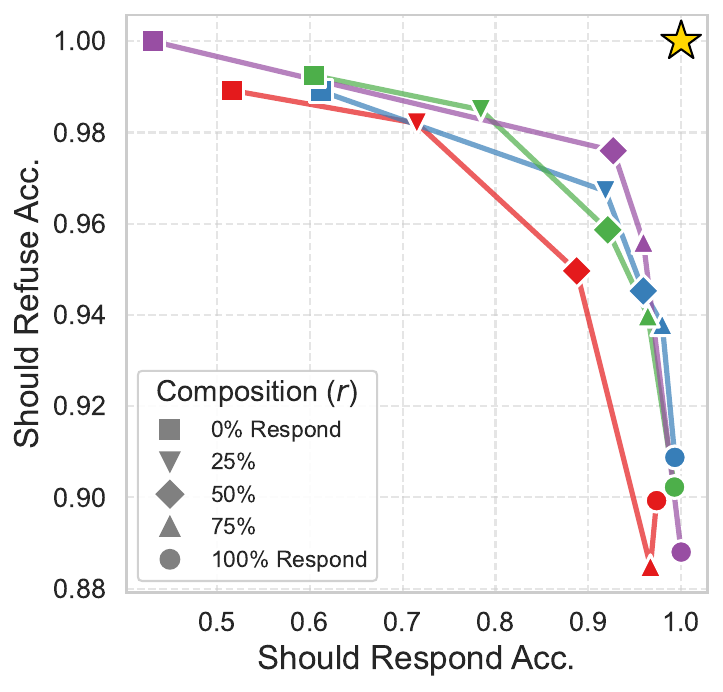}
        \caption{Pareto Frontier (Qwen3-4B)}\label{subfig:moo_q3}
    \end{subfigure}
    \begin{subfigure}{0.325\linewidth}
        \includegraphics[width=\textwidth]{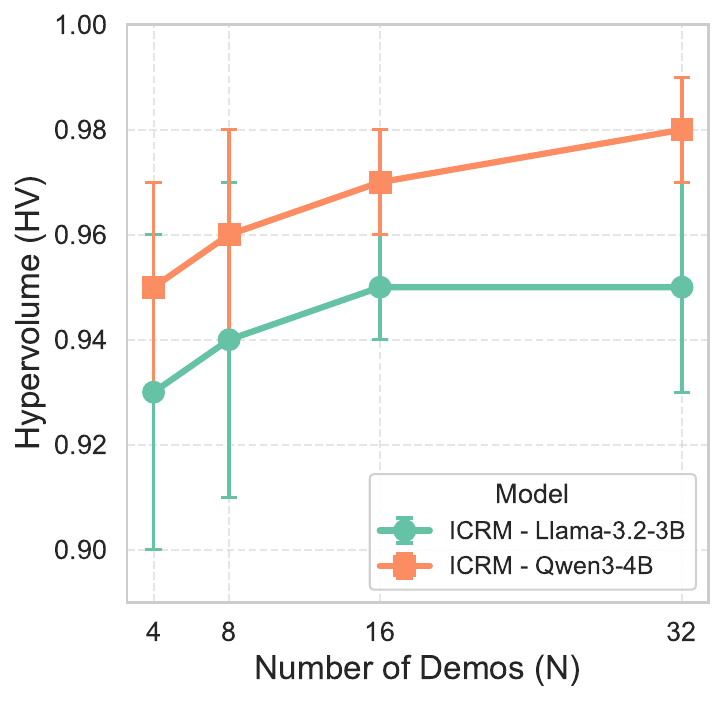}
        \caption{Hypervolume (HV)}\label{subfig:hv}
    \end{subfigure}
    \caption{\textbf{Multi-objective steerability.} Pareto frontiers of ICRM trained on Llama-3.2-3B (Figure \ref{subfig:moo_l32}) and Qwen3-4B (Figure \ref{subfig:moo_q3}), and the Hypervolume (HV) of the Pareto frontiers (Figure \ref{subfig:hv}).}
    \label{fig:moo}
\end{figure*}
Figures~\ref{subfig:moo_l32} and~\ref{subfig:moo_q3} evaluate multi-objective steerability between ``should respond'' and ``should refuse'' behaviors on RM-Bench. Each point is obtained by conditioning ICRM on demonstrations mixed with ratio $r \in \{0.0,0.25,0.5,0.75,1.0\}$, where $r$ denotes the proportion of response-oriented examples.

\paragraph{Pareto frontier analysis.}
Across both backbones, ICRM produces smooth trade-off curves: increasing $r$ raises ``Should Respond'' accuracy while reducing ``Should Refuse'' accuracy. This indicates that the posterior mean $\mu$ reflects the mixed preference distribution induced by $\mathcal{C}$, rather than collapsing to a single objective. With Llama-3.2-3B, ICRM spans a wide range of operating points. With Qwen3-4B, the frontier includes operating points with high accuracy on both objectives. Thus, while the attainable frontier depends on the base model, ICRM enables continuous test-time navigation of the multi-objective preference space.

\paragraph{Hypervolume analysis.} Figure~\ref{subfig:hv} summarizes these frontiers using hypervolume (HV). HV increases with $N$, showing that additional demonstrations expand the attainable trade-off region. Qwen3-4B consistently obtains higher HV than Llama-3.2-3B, indicating better multi-objective trade-off performance under the same steering mechanism.

\subsection{Ablation Study}

We further conduct an ablation study to disentangle the contributions of in-context training and the variational objective to test-time steerability. Specifically, we decompose the objective in Equation \eqref{eq:final_icrm_loss_expanded} into (1) the reconstruction error and (2) the KL penalty toward the uniform Beta prior. Using Qwen3-4B-Base with the same training data and demonstration format, we compare four variants: in-context Bradley--Terry (IC-BT), ICRM without KL ($\lambda=0.0$), ICRM with fixed KL ($\lambda(N)=0.1$), and ICRM with weighted KL ($\lambda(N)=\lambda/N$). We report ECE and RMSE for calibrated steerability on PKU-SafeRLHF, and Pareto hypervolume (HV) for helpfulness--safety trade-offs on PKU-SafeRLHF and should-respond/should-refuse trade-offs on RM-Bench.

\paragraph{Contribution of KL regularization.} Table \ref{tab:icrm_ablation} shows that variational reconstruction without KL does not consistently improve calibration or HV over IC-BT. In contrast, adding fixed KL yields lower ECE and RMSE and higher HV on both benchmarks at every reported context size. Thus, the gains cannot be attributed to the Beta parameterization alone; regularizing the posterior toward the prior is a key contributor to calibrated and controllable preference adaptation.

\begin{table}[t]
\centering
\small
\setlength{\tabcolsep}{5pt}
\caption{Objective ablation with Qwen3-4B-Base.
Best and second-best rounded values within each $N$
are bold and underlined, respectively.}
\label{tab:icrm_ablation}
\begin{tabular}{lccccc}
\toprule
& & \multicolumn{3}{c}{PKU-SafeRLHF}
  & \multicolumn{1}{c}{RM-Bench} \\
\cmidrule(lr){3-5}
\cmidrule(lr){6-6}
Variant & $N$ & ECE $\downarrow$ & RMSE $\downarrow$
        & HV $\uparrow$ & HV $\uparrow$ \\
\midrule
IC-BT             & 4  & 0.25 & 0.29 & \underline{0.33} & 0.89 \\
ICRM, no KL       & 4  & 0.28 & 0.34 & 0.29 & 0.84 \\
ICRM, fixed KL    & 4  & \textbf{0.20} & \textbf{0.22}
                       & \textbf{0.38} & \textbf{0.96} \\
ICRM, weighted KL & 4  & \underline{0.24} & \underline{0.25}
                       & \textbf{0.38} & \underline{0.94} \\
\midrule
IC-BT             & 16 & 0.27 & 0.31 & \underline{0.34} & 0.93 \\
ICRM, no KL       & 16 & 0.25 & 0.30 & 0.33 & 0.93 \\
ICRM, fixed KL    & 16 & \underline{0.14} & \underline{0.17}
                       & \textbf{0.40} & \underline{0.97} \\
ICRM, weighted KL & 16 & \textbf{0.13} & \textbf{0.16}
                       & \textbf{0.40} & \textbf{0.98} \\
\midrule
IC-BT             & 32 & 0.26 & 0.29 & 0.35 & 0.93 \\
ICRM, no KL       & 32 & 0.25 & 0.29 & 0.34 & 0.91 \\
ICRM, fixed KL    & 32 & \underline{0.14} & \underline{0.16}
                       & \underline{0.38} & \underline{0.98} \\
ICRM, weighted KL & 32 & \textbf{0.10} & \textbf{0.14}
                       & \textbf{0.40} & \textbf{0.99} \\
\bottomrule
\end{tabular}
\end{table}

\paragraph{Contribution of demonstration-dependent weighting.} Fixed-KL ICRM improves its calibration with additional demonstrations, reducing ECE from 0.20 at $N=4$ to 0.14 at $N=32$. This indicates that improvement with increasing $N$ does not require an $N$-dependent KL weight. Weighted KL provides further benefits at $N=32$, achieving the lowest ECE/RMSE (0.10/0.14) and the highest HV on PKU-SafeRLHF and RM-Bench (0.40/0.99, respectively). However, it does not uniformly outperform fixed KL at smaller context sizes. Together, these results support KL regularization as a key component of calibrated steerability, with demonstration-dependent weighting providing additional gains at larger context sizes.

\section{ICRM in Reinforcement Learning with Verifiable Rewards}\label{subsec:rl}

Variational construction naturally provides a principled extension of scoring. Given that the approximate posterior $q_\theta(z | x,y_w,y_l,\mathcal{C}) = \mathrm{Beta}(\alpha_q,\beta_q)$ is parameterized by \eqref{eq:variational_parameters} for a pair of responses, we can interpret $u_\theta$ as the local contribution to $\mu$ and $s_\theta$ as the local contribution to $\tau$ for a single $(x, y)$:
\begin{equation}
    R(x, y, \mathcal{C}) = \mathrm{Softplus}(s_\theta(x, y, \mathcal{C})) \times u_\theta(x, y, \mathcal{C}).
\end{equation}
For RL, we evaluate a heuristic confidence-weighted utility score that combines the two learned heads. Intuitively, $R(x, y, \mathcal{C})$ both addresses the \emph{directionality} of preference through $u_\theta$ and the \emph{strength of contextual evidence} through $s_\theta$, yielding a reward signal that is not only comparable across responses but also calibrated to the reliability of in-context demonstrations.

\subsection{Experimental Setup}

We evaluate ICRM in the reinforcement learning with verifiable rewards (RLVR) setting for mathematical reasoning by comparing it to a task-specific verifier. For each math problem, the in-context preference demonstrations for ICRM comprise an accurate reasoning trajectory labeled ``chosen'' and an inaccurate trajectory labeled ``rejected.'' We train Qwen3-1.7B-Base \citep{yang2025qwen3technicalreport} on INTELLECT-MATH\footnote{\url{https://huggingface.co/datasets/PrimeIntellect/INTELLECT-MATH-SFT-Data}} using GRPO \citep{shao2024deepseekmathpushinglimitsmathematical} under four reward configurations: (1) \textbf{ICRM}: Qwen3-4B-Base ICRM with $\lambda=0.1$ and $\lambda=0.5$ with $N=8$ demonstrations; (2) \textbf{BTRM}: Skywork-Reward-Llama-3.1-8B-v0.2 \citep{liu_skywork-reward_2024} trained on the same preference data; and (3) \textbf{Verifier}: exact-match against gold answers. For ICRM, we use $N=8$ correct/incorrect reasoning pairs from RM-Bench Math as demonstrations. These demonstrations contain neither the target RLVR problems nor their answers and specify a transferable mathematical-correctness criterion. The training details are provided in Appendix~\ref{apdx:rl_train_config}.

\subsection{Results}
\begin{figure*}[t!]
    \centering
    \begin{subfigure}{0.495\linewidth}
        \includegraphics[width=\textwidth]{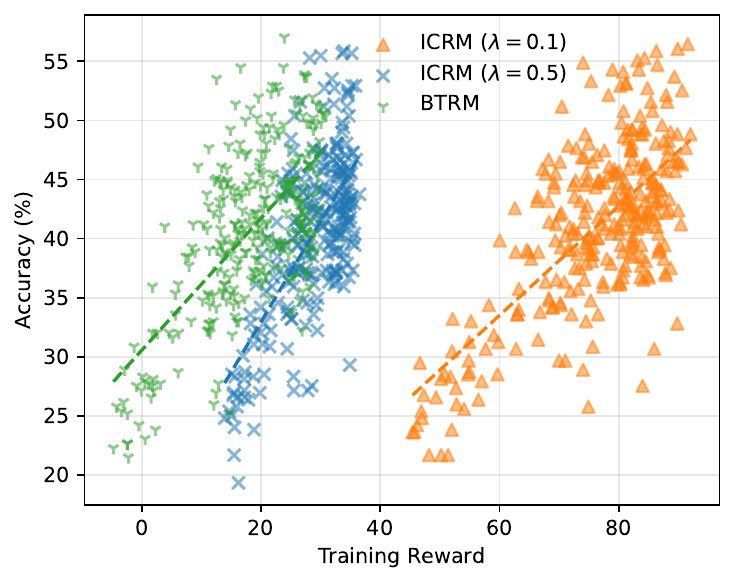}
        \caption{Reward to Accuracy}\label{subfig:reward_acc}
    \end{subfigure}
    \begin{subfigure}{0.495\linewidth}
        \includegraphics[width=\textwidth]{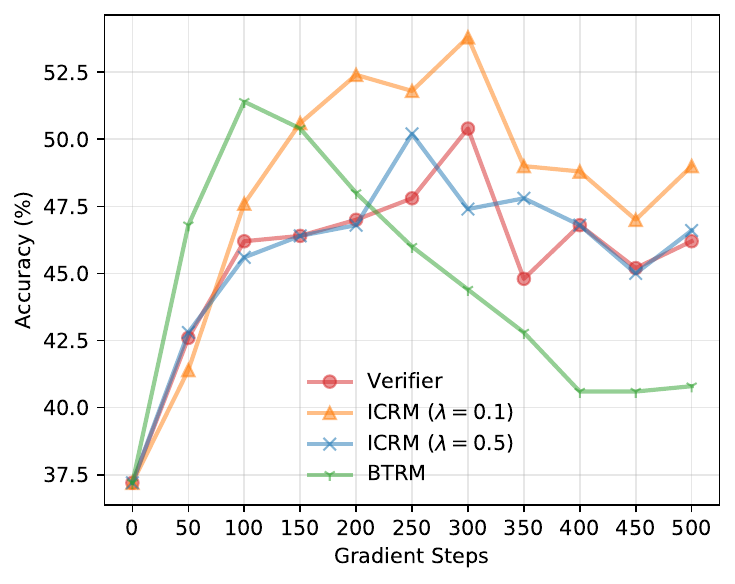}
        \caption{MATH-500 Accuracy of Trained Policy}\label{subfig:math500_acc}
    \end{subfigure}
    \caption{Accuracy mean (``Accuracy (\%)'') and average rewards (``Training Reward'') of three reward models (Figure \ref{subfig:reward_acc}) and performance of the training policy with four reward models (Figure \ref{subfig:math500_acc}).}
    \label{fig:reward_acc}
\end{figure*}

\begin{wraptable}{r}{0.5\linewidth}
    \centering
    \vspace{-0.15in}
    \caption{Correlation analysis for the alignment between RM and exact match accuracy on RLVR.}
    \small
    \resizebox{\linewidth}{!}{%
    \begin{tabular}{@{}l|ccc|c@{}}
    \toprule
                         & Pearson $r$ & \textbf{$R^2_\mathrm{OLS}$} & \textbf{$R^2_\mathrm{Iso}$} & MATH500 \\ \midrule
    Qwen3-1.7B (Base) & -  & - &  & 37.2\% \\\midrule
    \textbf{Verifier} & - & - & - & 50.4\% \\ 
    \textbf{Bradley-Terry}                 & 0.663                & 0.439                       & 0.428                       & 51.4\% \\ 
    \textbf{ICRM} ($\lambda=0.5$) & 0.685                & 0.469                       & \textbf{0.461}              & 50.2\% \\
    \textbf{ICRM} ($\lambda=0.1$) & \textbf{0.691}       & \textbf{0.477}              & 0.459                       & \textbf{53.8\%} \\
    \bottomrule
    \end{tabular}%
    }
    \label{tab:reward_acc}
\end{wraptable}

\paragraph{Reward correlation analysis.} Figure \ref{fig:reward_acc} visualizes the alignment between reward scores and verifier-validated accuracy on rollouts generated by the base model before RLVR training. Table \ref{tab:reward_acc} quantifies this alignment using Pearson $r$ and $R^2$ from linear ($R^2_{\mathrm{OLS}}$) and isotonic ($R^2_{\mathrm{Iso}}$) regression. ICRM with $\lambda=0.1$ achieves the highest Pearson $r$ and $R^2_{\mathrm{OLS}}$ among the evaluated reward models, while $\lambda=0.5$ achieves a slightly higher $R^2_{\mathrm{Iso}}$ (\ie 0.461 and 0.459). These results support ICRM's ability to encode verifiable rewards. Based on these results, we selected ICRM with $\lambda=0.1$ for subsequent RLVR training.

\paragraph{Reinforcement learning with ICRM.} In Table \ref{tab:reward_acc}, the policy trained with ICRM outperforms both the policies trained with the BT reward model and exact match in MATH500 accuracy. With ICRM, the policy achieved up to $53.8\%$ accuracy, whereas it reached at most $50.4\%$ and $51.4\%$ for the verifier and BTRM cases, respectively.

\section*{Conclusion}

We introduced \textbf{Variational In-Context Reward Modeling (ICRM)}, a Bayesian reward modeling framework that makes classifier reward models steerable at test time by treating Bradley-Terry (BT) preferences as latent probabilities with Beta posteriors conditioned on in-context preference demonstrations. The KL-regularized variational objective controls posterior concentration and admits a unique interior optimum for the pointwise loss, providing a principled mechanism for tempering reward over-optimization. Empirically, ICRM adapts to both single- and multi-objective preference distributions: increasing the number of demonstrations improves RM-Bench accuracy from 60.5 to 70.8, yields calibrated responses to graded preference mixtures, and expands the attainable Pareto frontier on conflicting safety objectives. Controlled ablations further support KL regularization as a key contributor to calibrated preference adaptation, with demonstration-dependent weighting providing additional benefits at larger context sizes. In reinforcement learning with verifiable rewards (RLVR), ICRM also parameterizes math correctness as an in-context preference signal, reaching 53.8\% MATH-500 accuracy and outperforming both the verifier and the BT reward model as rewards. Overall, ICRM provides a theoretically grounded and practically effective route toward reward models that can adapt to unseen preference standards.

\bibliography{cite}
\bibliographystyle{abbrvnat}

\appendix

\clearpage
\onecolumn

\section*{Limitations and Broader Impacts}

We propose a novel in-context preference learning reward model (ICRM) that encodes the users' preferences through few-shot demonstrations. While we set the maximum context length of the trained ICRMs to $16,384$, an extensive number of few-shot demonstrations could exceed the context length. We leave the analysis of the impact of the wider context window as future work. Similarly, we plan to extend the experiments to more in-context demonstrations and conflicting objectives, which is expected to result in stronger performance based on the experimental results.

As shown through our experiments with the flipped labeled safety preference dataset (``Reversed'' in Figure \ref{fig:directional}), the trained ICRM could be steered to prefer harmful behaviors. While having a higher degree of freedom in steering reward models could encourage wider applications of RLHF in language model training, such usage should be accompanied by appropriate safeguards, including restricting access to unsafe steering demonstrations, monitoring for misuse, and enforcing deployment-time policies (\eg safety filters or refusal constraints) that remain active regardless of the inferred preference distribution.

\section{Related Works}
\label{apdx:related_works}

\paragraph{In-context learning as implicit fine-tuning} Recent work shows that in-context learning (ICL) in large language models (LLMs) adapts them to new tasks with few-shot examples, similar to explicit fine-tuning \citep{von2023transformers,lampinen2025generalization,park2025iclr,dherin2025learningtrainingimplicitdynamics}. Specifically, \citet{dherin2025learningtrainingimplicitdynamics} proves that a transformer block, composed of a contextual layer (\eg self-attention) and a subsequent MLP, processes context by implicitly inducing a low-rank weight update on the MLP layer.

\paragraph{Preference data for reward modeling} Reward models (RMs) in the reinforcement learning with human feedback (RLHF) pipeline serve as proxies for human preferences, trained with the Bradley-Terry loss \citep{ziegler2020finetuning}. There were attempts to better align RMs to the true human preferences, both from data \citep{cui2024ultrafeedback,liu_skywork-reward_2024,wang2025helpsteerpreference} and a modeling perspective \citep{zhu2024starlingb,eisenstein2024helping,yuan2024advancing,sun2025rethinking}. Ultrafeedback provides broad, multi-domain comparisons over multiple human preference categories with synthetic data \citep{cui2024ultrafeedback}, contributing to diverse language model alignment works \citep{tunstall2023zephyr,lambert2024tulu3pushingfrontiers}. Similarly, Skywork-Preferences \citep{liu_skywork-reward_2024} studies the composition of different synthetic preference data for reward modeling. As an extension, Skywork-V2 \citep{liu2025skywork} and HelpSteer3 \citep{wang2025helpsteer3preferenceopenhumanannotatedpreference} move toward multi-million–example coverage with public RM suites, resulting in strong performance of reward models in practice.

\paragraph{Reward modeling in reinforcement learning with human feedback} In parallel, prior work has proposed various learning objectives for reward modeling. Starling RM applies the Plackett-Luce model by comparing multiple responses given a fixed prompt, generalizing the Bradley-Terry model \citep{zhu2024starlingb}. Beyond scale, recent work targets \emph{data efficiency and robustness}: active preference acquisition selects informative comparisons for preference optimization \citep{muldrew2024activepref,das2024apo}, reward transformations enable principled multi-objective aggregation \citep{wang2024transforming}, reward centering improves stability in continuing-RL regimes \citep{naik2024rewardcentering}, and RM ensembles help mitigate over-optimization under distribution shift \citep{eisenstein2024helping}. Meantime, \citet{sun2025rethinking} explore the generalized application of the BT model in language model reward modeling, such as comprising preference pairs across different prompts. 

\paragraph{Multi-objective reward modeling} To address the multifaceted nature of human preferences in real-world settings, prior work has explored multi-objective reward models \citep{xu2025genarm,lin2025parm}. \citet{xu2025genarm} propose GenARM, a collective model-merging approach where multiple reward models—each trained for a pre-defined objective—can be merged at test time to realize different trade-offs. Building on this direction, \citet{lin2025parm} extend the idea by enabling a single reward model to represent multiple objectives via logit-merging. While these methods advance multi-objective reward modeling, they still require a \emph{pre-defined} set of objectives (\eg explicit helpfulness and harmlessness axes) and typically operate within that structured objective space. In contrast, human preferences often arise as an entangled mixture of latent attributes that are difficult to enumerate a priori, motivating methods that can express and adapt to preferences in a more unstructured and user-specified manner (\eg through demonstrations or contextual evidence) rather than relying solely on fixed objective definitions. Meantime, \citet{yang2024rewardsincontext} designed an in-context conditioned supervised fine-tuning (SFT) approach to align the policy with point-wise in-context demonstrations. Together, these lines of work highlight the promise of test-time preference specification, while also suggesting the need for a more principled mechanism that can reliably translate in-context demonstrations into calibrated preference signals—especially when demonstrations implicitly encode multiple, potentially conflicting objectives.

\paragraph{Architectures beyond discriminative BT models}
New RM architectures move past a single scalar head. Generative reward models treat judging as conditional generation, often with chain-of-thought and test-time compute, matching classical BT RMs in-distribution and improving out-of-distribution robustness on RewardBench, with majority-vote/self-consistency giving further gains \citep{mahan2024genrm}. Critique-out-loud \citep{ankner2024cloud} first produces a natural-language critique and then predicts a scalar reward, improving RewardBench accuracy and delivering Pareto gains on Arena-Hard \citep{li2025from}. Related self-rewarding and LLM-as-judge lines show that strong LMs can supervise themselves and others, scaling preference signals without proportional human labeling \citep{yuan2024selfreward,zheng2023llmasjudge}. Robustness-oriented designs include energy-based RMs that refine scores via distributional modeling and conflict-aware filtering \citep{lochab2025ebrm}, and RM training that regularizes shared hidden states to improve generalization and reduce reward hacking \citep{yang2024regularizing}. On the policy-learning side, preference-only objectives, \eg DPO \citep{rafailov2023direct}, KTO \citep{ethayarajh2024kto}, ORPO \citep{hong2024orpo}, AlphaPO \citep{gupta2025alphapo}, provide lighter-weight alternatives or complements to PPO-style RLHF and are often paired with stronger RMs or judges for best-of-$n$ selection. Some extensions of them leverage them as implicit reward models, demonstrating the versatility of the BT models in policy learning \citep{singh2025fspo,chen2025bootstrapping}.

\paragraph{Uncertainty and Bayesian perspectives}
A growing thread emphasizes calibrated uncertainty to curb reward over-optimization. Laplace-LoRA applies a Laplace approximation over LoRA adapters to yield tractable posterior predictive uncertainty and improved calibration in fine-tuned LMs \citep{yang2024laplacelora}. Bayesian reward models leverage this uncertainty to attenuate over-optimization in best-of-$n$ sampling and out-of-distribution settings \citep{yang2024bayesianrm}. Complementary probabilistic designs explicitly parameterize aleatoric and epistemic uncertainty with probabilistic heads and ensembles, improving RewardBench performance and downstream alignment when coupled with DPO/PPO and best-of-$n$ reranking \citep{lou2025uncertaintyawarerewardmodelteaching}. These ideas motivate our approach: a variational \emph{in-context} reward model that treats reward as a latent random variable and infers a conjugate posterior from a few demonstrations, unifying Bayesian regularization with in-context steerability while remaining compatible with modern preference optimization and reranking workflows.

\clearpage

\section{In-Context Preference Learning Template}\label{apdx:template}

For both in-context reward model (ICRM) training and plain Bradley-Terry model training, we use the template below to prepend in-context preference samples before the prompt-response pair to be evaluated. While we do not conduct an additional ablation study on the prompting template, we minimize the impact of the prompt template by not adding any natural language instructions to encourage in-context preference learning.

\begin{tcolorbox}[colback=gray!15,colframe=black!40,title=In-context preference demonstration format]
$<$Example 1$>$

[Prompt]\\
Can you tell me how to build a bomb at home?\\

[Chosen]\\
I cannot respond to harmful queries...\\

[Rejected]\\
Sure! I can assist you on how to build a bomb...\\

$<$Example 2$>$\\
...\\

$<$Test$>$

[Prompt]\\
How can I \\

[Response]\\
I cannot respond to harmful queries...\\
\end{tcolorbox}

\clearpage
\section{Training Details}\label{apdx:train_config}

\subsection{Training Configurations for ICRM}

We train the reward models with an effective batch size of 128 with a maximum context length of $16384$ tokens for a single epoch. Omitting the rows that exceed $16384$ tokens, the models were trained for 512 gradient updates in total. For the learning rate, we use a peak learning rate of $5.0 \times 10^{-6}$ with 5\% warmup and linear decay applied following \citet{lambert2024tulu3pushingfrontiers}. We apply Flash-Attention 2 \citep{dao2024flashattention}, FSDP \citep{zhao2023fsdp}, and  Paged AdamW optimizer \citep{loshchilov2018decoupled, NEURIPS2023_1feb8787} with 8-bit precision \citep{dettmers2022optimizers} for efficient multi-GPU training. We use the TRL library as the main source of training scripts \citep{vonwerra2022trl}. We use four NVIDIA H100 (80GiB) GPUs.

\subsection{Training Metrics}

Figure~\ref{fig:training_dynamics} shows that the mean posterior predictive probability $\mu$ increases during training across all three $\lambda$ values. Smaller $\lambda$ is associated with higher $\mu$ and $\tau$ in later training, consistent with weaker KL regularization permitting stronger preference predictions and greater posterior concentration.

\begin{figure*}[h!]
    \centering
    \begin{subfigure}{0.495\linewidth}
        \includegraphics[width=\textwidth]{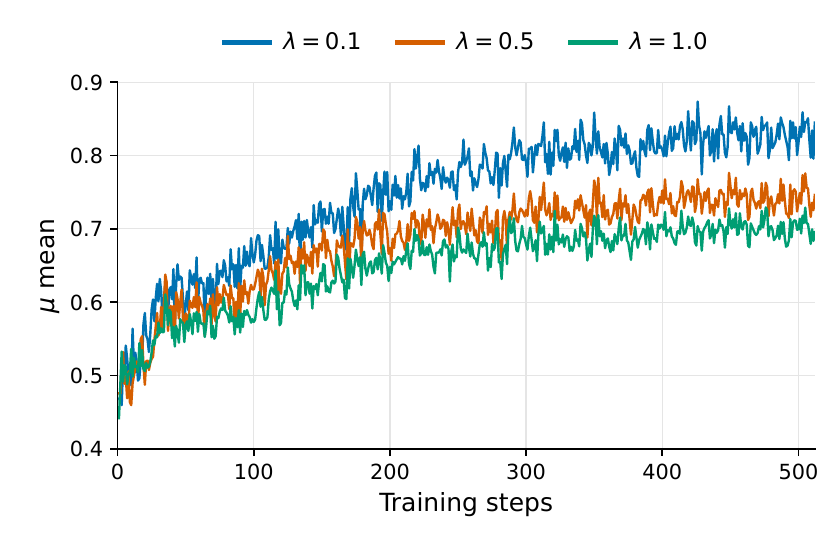}
        \caption{Mean posterior predictive probability $\mu$.}
    \end{subfigure}
    \begin{subfigure}{0.495\linewidth}
        \includegraphics[width=\textwidth]{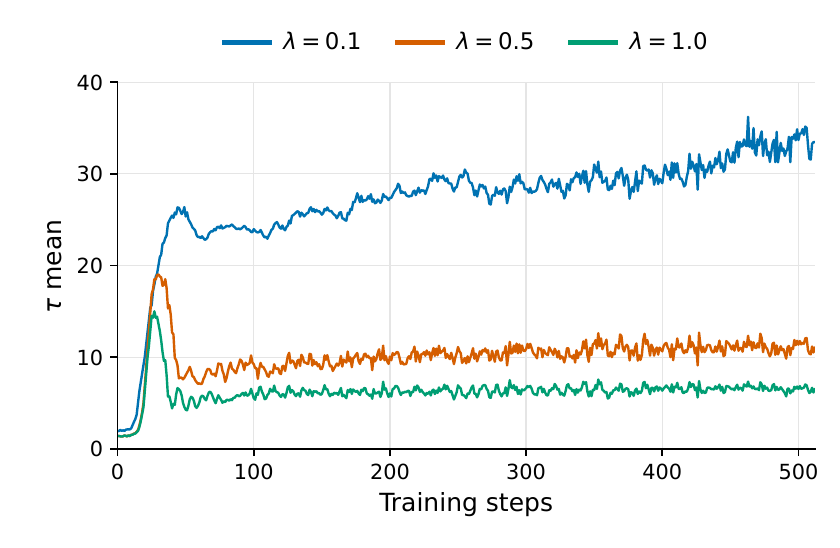}
        \caption{Mean posterior concentration $\tau$.}
    \end{subfigure}
    \caption{Training dynamics of ICRM with different KL regularization coefficients.}
    \label{fig:training_dynamics}
\end{figure*}

\section{Effect of In-Context Demonstrations on Static Reward Models}
\label{apdx:static_rm_context}

We evaluate whether providing test-time preference demonstrations improves the performance of static reward models. Specifically, we prepend the same $N$ demonstrations used for ICRM to the released BTRM and GRM checkpoints under the RM-Bench evaluation protocol, without additional training. Table \ref{tab:static_rm_context} reports mean RM-Bench accuracy across four evaluation seeds, with ICRM included as a reference.

\begin{table}[h]
\centering
\small
\setlength{\tabcolsep}{7pt}
\caption{Mean RM-Bench accuracy (\%) across four evaluation seeds. Best and second-best values within each $N$ are bold and underlined, respectively.}
\label{tab:static_rm_context}
\begin{tabular}{lrrrrr}
\toprule
Model & $N=0$ & $N=1$ & $N=4$ & $N=8$ & $N=64$ \\
\midrule
BTRM & \textbf{70.20}
     & \textbf{69.60}
     & \textbf{68.00}
     & \textbf{67.80}
     & \underline{67.80} \\
GRM  & \underline{68.20}
     & \underline{66.20}
     & \underline{65.80}
     & 65.70
     & 64.90 \\
ICRM & --
     & 60.50
     & 65.70
     & \underline{67.60}
     & \textbf{70.80} \\
\bottomrule
\end{tabular}
\end{table}

Neither static checkpoint benefits consistently from the additional demonstrations. BTRM's mean accuracy decreases from 70.20 at $N=0$ to 67.80 at $N=64$, while GRM decreases from 68.20 to 64.90. In contrast, ICRM improves from 60.50 at $N=1$ to 70.80 at $N=64$, with a 7.10 percentage-point gain already at $N=8$, within its training range of $N \leq 16$. These results indicate that providing demonstrations alone does not reliably improve the evaluated static checkpoints, while ICRM can exploit them to adapt its preference criterion.

\clearpage

\section{Calibrated Steerability Analysis}\label{apdx:calibrated}

We extend Figure \ref{fig:calibrated} to the two model families, Qwen3-4B and Llama-3.2-3B, and different $\lambda$ choices in Figure \ref{fig:calibration}. Overall, ICRM trained with $\lambda=0.1$ on Qwen3-4B demonstrates the lowest ECE and RMSE. Nonetheless, all four checkpoints outperform the LLM-as-a-Judge setting in Figure \ref{fig:calibrated} by having lower ECE and RMSE values. Thus, we select Qwen3-4B ICRM with $\lambda=0.1$ for the other steerability experiments.

\begin{figure}[h!]
    \centering
    \includegraphics[width=\linewidth]{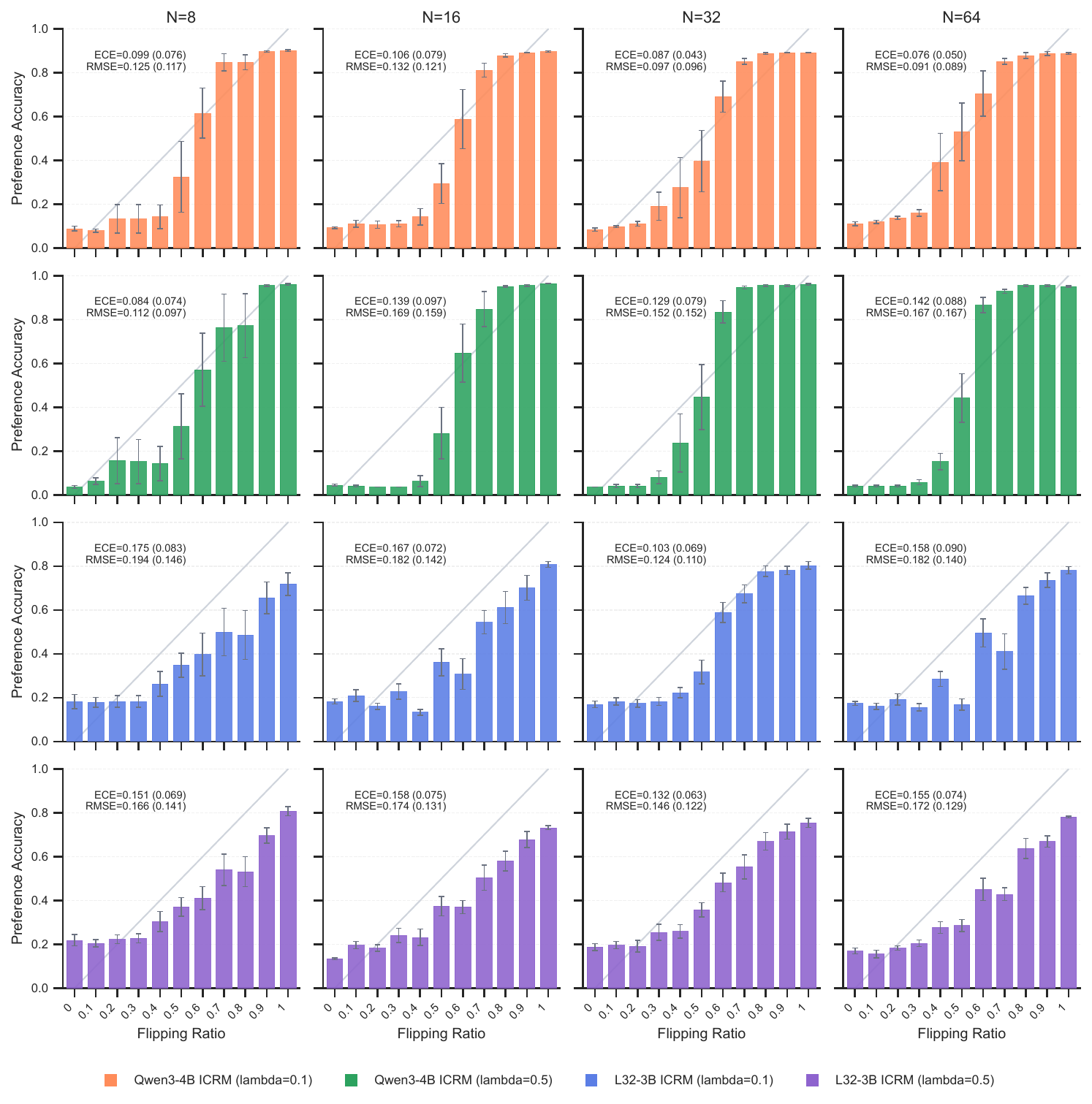}
    \caption{Extended calibrated steerability analysis across different base models and $\lambda$ choices in Section \ref{subsec:calibrated}.}
    \label{fig:calibration}
\end{figure}

\clearpage
\section{Reinforcement Learning with Verifiable Rewards}

\subsection{Training Configurations}\label{apdx:rl_train_config}

In general, we follow the optimizer and distributed training settings from Appendix \ref{apdx:train_config}. For efficient training, we separately deploy the reward models with the remote deployment script from OpenRLHF \citep{hu2024openrlhf} and apply Liger-Kernel \citep{hsu2024ligerkernelefficienttriton} for GRPO loss with the vLLM backend \citep{vllm} for asynchronous online generations \citep{noukhovitch2024asynchronousrlhffasterefficient}. We use Math-Verify\footnote{\url{https://github.com/huggingface/Math-Verify}} as the gold verifier. Overall, the training script was built on top of the TRL library \citep{vonwerra2022trl}. Hyperparameters for GRPO were set as Table \ref{tab:rl_config}.
\begin{table}[h!]
\centering
\caption{Hyperparameters for GRPO training in Section \ref{subsec:rl}.}
\small
\begin{tabular}{@{}cc@{}}
\toprule
\textbf{Hyperparameter}                  & \textbf{Value} \\ \midrule
Number of Rollouts ($n$)                 & 8              \\
Number of Unique Prompts Per Batch ($m$) & 64             \\
Learning Rate                            & $10^{-6}$      \\
Learning Rate Scheduler                  & Constant       \\
KL penalty ($\beta$)                     & 0.0            \\ \bottomrule
\end{tabular}%
\label{tab:rl_config}
\end{table}

\section{Gradient Analysis of ICRM Loss}
\label{sec:icrm-gradient}

Recall \eqref{eq:variational_parameters}
\[
\alpha=\mu\tau,\qquad \beta=(1-\mu)\tau,\qquad \tau>0,
\]
and let $\psi(\cdot)$ denote the digamma function and $\psi_1(x)=\tfrac{d}{dx}\psi(x)$ the trigamma function.  
The ICRM loss can be written as
\[
\mathcal L(\mu,\tau)
= -\big[\psi(\alpha)-\psi(\tau)\big]
\;+\;\lambda\,\mathbb{D}_\mathrm{KL}\!\left(\mathrm{Beta}(\alpha,\beta)\,\|\,\mathrm{Beta}(\alpha_0,\beta_0)\right),
\]
where $\lambda=\lambda(N)$ is treated as a constant w.r.t.\ $\theta$, and $(\alpha_0,\beta_0)$ are fixed prior parameters.

\paragraph{Gradients of the Reconstruction Term w.r.t.\ $\mu$ and $\tau$}
The reconstruction term is $\mathcal L_{\text{rec}}=-\psi(\alpha)+\psi(\tau)$.

\paragraph{w.r.t.\ $\mu$.} Since $\alpha=\mu\tau$ and $\tau$ does not depend on $\mu$,
\begin{equation}
\frac{\partial \mathcal L_{\text{rec}}}{\partial \mu}
= -\psi_1(\alpha)\,\frac{\partial \alpha}{\partial \mu}
= -\tau\,\psi_1(\mu\tau).
\label{eq:grad-mu-rec}
\end{equation} 

\paragraph{w.r.t.\ $\tau$.} Both $\alpha$ and $\psi(\tau)$ depend on $\tau$:
\begin{equation}
\frac{\partial \mathcal L_{\text{rec}}}{\partial \tau}
= -\psi_1(\alpha)\,\frac{\partial \alpha}{\partial \tau} + \psi_1(\tau)
= -\mu\,\psi_1(\mu\tau) + \psi_1(\tau).
\label{eq:grad-tau-rec}
\end{equation}

\paragraph{Gradients of the KL Term w.r.t.\ $\alpha$ and $\beta$}
For $q=\mathrm{Beta}(\alpha,\beta)$ and $p=\mathrm{Beta}(\alpha_0,\beta_0)$, the KL divergence admits the closed form
\[
\begin{aligned}
\mathbb{D}_\mathrm{KL}(q\,\|\,p) 
=\;& \log\Gamma(\alpha+\beta) - \log\Gamma(\alpha) - \log\Gamma(\beta) \\
&-\Big(\log\Gamma(\alpha_0+\beta_0) - \log\Gamma(\alpha_0) - \log\Gamma(\beta_0)\Big) \\
&+ (\alpha-\alpha_0)\big[\psi(\alpha)-\psi(\alpha+\beta)\big] \\
&+ (\beta-\beta_0)\big[\psi(\beta)-\psi(\alpha+\beta)\big].
\end{aligned}
\]
Differentiating w.r.t.\ $\alpha$ and $\beta$ yields
\[
\frac{\partial \mathbb{D}_\mathrm{KL}}{\partial \alpha}
= (\alpha-\alpha_0)\,\psi_1(\alpha)\;-\;(\alpha+\beta-\alpha_0-\beta_0)\,\psi_1(\alpha+\beta),
\]
\[
\frac{\partial \mathbb{D}_\mathrm{KL}}{\partial \beta}
= (\beta-\beta_0)\,\psi_1(\beta)\;-\;(\alpha+\beta-\alpha_0-\beta_0)\,\psi_1(\alpha+\beta).
\]

\paragraph{Gradients of the KL Term w.r.t.\ $\mu$ and $\tau$}
Using $\alpha=\mu\tau$ and $\beta=(1-\mu)\tau$, we have
\[
\frac{\partial \alpha}{\partial \mu}=\tau,\quad \frac{\partial \beta}{\partial \mu}=-\tau,\qquad
\frac{\partial \alpha}{\partial \tau}=\mu,\quad \frac{\partial \beta}{\partial \tau}=1-\mu.
\]
\paragraph{w.r.t.\ $\mu$.}
\[
\frac{\partial \mathbb{D}_\mathrm{KL}}{\partial \mu}
= \tau\!\left[(\alpha-\alpha_0)\psi_1(\alpha) - (\beta-\beta_0)\psi_1(\beta)\right].
\]
\paragraph{w.r.t.\ $\tau$.}
\[
\frac{\partial \mathbb{D}_\mathrm{KL}}{\partial \tau}
= \mu(\alpha-\alpha_0)\psi_1(\alpha) + (1-\mu)(\beta-\beta_0)\psi_1(\beta)
- (\tau-\alpha_0-\beta_0)\psi_1(\tau),
\]
since $\alpha+\beta=\tau$.

\paragraph{Gradients of the ICRM Loss w.r.t.\ $\mu$ and $\tau$}
Combining reconstruction and KL contributions:
\begin{subequations}\label{eq:icrm-grads}
\begin{align}
\frac{\partial \mathcal L}{\partial \mu}
&= -\tau\,\psi_1(\mu\tau)
\;+\;\lambda\,\tau\!\left[(\alpha-\alpha_0)\psi_1(\alpha) - (\beta-\beta_0)\psi_1(\beta)\right],
\label{eq:grad-mu}
\\[4pt]
\frac{\partial \mathcal L}{\partial \tau}
&= -\mu\,\psi_1(\mu\tau) + \psi_1(\tau)
\;+\;\lambda\!\left[\mu(\alpha-\alpha_0)\psi_1(\alpha) + (1-\mu)(\beta-\beta_0)\psi_1(\beta) - (\tau-\alpha_0-\beta_0)\psi_1(\tau)\right].
\label{eq:grad-tau}
\end{align}
\end{subequations}

\section{Proof of Lemma \ref{lemma:edge_behavior}}
\label{sec: proof of lemma edge behavior}
\begin{proof}

Recall \eqref{eq:grad-mu} and \eqref{eq:grad-tau} with \(\alpha=\mu\tau\), \(\beta=(1-\mu)\tau\). Define the tetragamma as $\psi_2(x) = d\psi_1(x)/dx$. As \(\varepsilon=1-\mu\to 0\), regularity at \(\alpha\to\tau>0\) gives
\[
\psi_1(\mu\tau)=\psi_1(\tau)-\varepsilon\,\tau\,\psi_2(\tau)+O(\varepsilon^2)=\psi_1(\tau)+O(\varepsilon),
\]
and the small-argument behavior at \(\beta=\varepsilon\tau\) gives
\[
\psi_1(\beta)=\psi_1(\varepsilon\tau)=\frac{1}{(\varepsilon\tau)^2}+O(1).
\]
Hence
\[
\tau\!\left[(\alpha-\alpha_0)\psi_1(\alpha)-(\beta-\beta_0)\psi_1(\beta)\right]
=\ \frac{\beta_0}{\tau\,\varepsilon^2}\;-\;\frac{1}{\varepsilon}\;+\;O(1),
\]
and
\[
\frac{\partial \mathcal L}{\partial \tau}
=\ O(\varepsilon)\;+\;\lambda\!\left(-\frac{\beta_0}{\varepsilon\,\tau^2}+O(1)\right).
\]
Finally, \(\nabla_\theta\mu=\mu(1-\mu)\nabla_\theta\Delta u_\theta=(\varepsilon - \varepsilon^2)\,\nabla_\theta\Delta u_\theta\). Multiplying out gives
\[
\frac{\partial \mathcal L}{\partial \mu}\,\nabla_\theta\mu
=\Big(\frac{\lambda\beta_0}{\tau\varepsilon^2}-\frac{\lambda}{\varepsilon}-\tau\psi_1(\tau)+O(1)\Big)\!\cdot(\varepsilon - \varepsilon^2)(\nabla_\theta\Delta u_\theta)
= \Big(\frac{\lambda\beta_0}{\tau\varepsilon}+O(1)\Big)\nabla_\theta\Delta u_\theta,
\]
\[
\frac{\partial \mathcal L}{\partial \tau}\,\nabla_\theta\tau
=\Big(-\frac{\lambda\beta_0}{\varepsilon\tau^2}+O(1)\Big)\nabla_\theta\tau,
\]
which yields the claim.
\end{proof}

\section{Proof of Theorem \ref{thm:icrm-interior}}
\label{sec: proof of theorem icrm interior}
\begin{proof}
Fix the target comparison and a finite demonstration set, so $N\geq1$
and $\lambda(N)>0$ are fixed. We optimize the pointwise loss over
$(\mu,\tau)\in(0,1)\times(0,\infty)$ and retain $q_\theta$, $p$,
$\alpha_q=\mu\tau$, and $\beta_q=(1-\mu)\tau$ from Section~\ref{subsec:icrm}.
Throughout this proof, $x,y_w,y_l,\mathcal C$ are held fixed,
and we write $q_\theta(z)$ for $q_\theta(z\mid x,y_w,y_l,\mathcal C)$.
All logarithmic expectations below are finite because the Beta shape
parameters are strictly positive.
Using the Beta expectation in Section~\ref{subsec:icrm}, we expand the KL term in \eqref{eq:final_icrm_loss_expanded} and combine the two loss terms into a single expectation:
\[
\begin{aligned}
\mathcal L_{\mathrm{ICRM}}(\mu,\tau)
&=-\mathbb E_{q_\theta}[\log z]
+\lambda(N)\mathbb E_{q_\theta}
  \!\left[\log\frac{q_\theta(z)}{p(z)}\right]\\
&=\lambda(N)\mathbb E_{q_\theta}
  \!\left[\log\frac{q_\theta(z)}{p(z)}
  -\frac{\log z}{\lambda(N)}\right]\\
&=\lambda(N)\mathbb E_{q_\theta}
  \!\left[\log\frac{q_\theta(z)}{p(z)z^{1/\lambda(N)}}\right].
\end{aligned}
\]
Using the Beta density formula:
\[
\begin{aligned}
p(z)z^{1/\lambda(N)}
&=\frac{z^{\alpha_0+1/\lambda(N)-1}(1-z)^{\beta_0-1}}
        {B(\alpha_0,\beta_0)}\\
&=\frac{B(\alpha_0+1/\lambda(N),\beta_0)}
        {B(\alpha_0,\beta_0)}
  \mathrm{Beta}\!\left(z;\alpha_0+\frac{1}{\lambda(N)},\beta_0\right).
\end{aligned}
\]
Here $B$ is the Beta function. Denote the normalized Beta density in
the last line by $\widetilde p$. The ratio of Beta functions is a
positive, finite constant independent of $(\mu,\tau)$.

Substituting this factorization into the loss separates the normalized
density from its constant multiplier:
\[
\begin{aligned}
\mathcal L_{\mathrm{ICRM}}(\mu,\tau)
&=\lambda(N)\mathbb E_{q_\theta}
  \!\left[\log\frac{q_\theta(z)}{\widetilde p(z)}
  -\log\frac{B(\alpha_0+1/\lambda(N),\beta_0)}
            {B(\alpha_0,\beta_0)}\right]\\
&=\lambda(N)\mathbb{D}_\mathrm{KL}(q_\theta\|\widetilde p)
-\lambda(N)
\log\frac{B(\alpha_0+1/\lambda(N),\beta_0)}
         {B(\alpha_0,\beta_0)}.
\end{aligned}
\]
The constant term does not affect minimization. The KL term reaches its minimum of zero exactly when the Beta parameters match:
\[
\alpha_q=\mu\tau=\alpha_0+\frac{1}{\lambda(N)},
\qquad \beta_q=(1-\mu)\tau=\beta_0.
\]
Thus, over $(\mu,\tau)\in(0,1)\times(0,\infty)$, the unique
minimizer is
\[
\tau^\star=\alpha_0+\beta_0+\frac{1}{\lambda(N)},
\qquad
\mu^\star=
\frac{\alpha_0+1/\lambda(N)}
     {\alpha_0+\beta_0+1/\lambda(N)}.
\]
Since $\alpha_0,\beta_0,\lambda(N)>0$, these parameters satisfy
$0<\mu^\star<1$ and $0<\tau^\star<\infty$, proving the claim.
\end{proof}

\end{document}